\documentclass{article}
\usepackage[utf8]{inputenc}
\usepackage{amsmath, amsthm}
\usepackage{bbm}
\usepackage{booktabs} 
\usepackage{natbib}
\usepackage{amsfonts}
\usepackage{amssymb}
\usepackage{algorithm2e}
\usepackage{caption}
\usepackage{graphicx}
\usepackage{subcaption}
\usepackage{color}
\usepackage{url}
\usepackage{hyperref}
\usepackage{breakurl}
\usepackage[capitalise,nameinlink,noabbrev]{cleveref}
\usepackage{todonotes}
\usepackage[accepted]{icml2020}
\SetKwRepeat{Repeat}{repeat\{}{until}%

\newtheorem{theorem}{Theorem}
\newtheorem{proposition}{Proposition}
\newtheorem{definition}{Definition}

\newtheorem{assumption}{Assumption}

\newtheoremstyle{TheoremNum}
        {\topsep}{\topsep}              
        {\itshape}                      
        {}                              
        {\bfseries}                     
        {.}                             
        { }                             
        {\thmname{#1}\thmnote{ \bfseries #3}}
    \theoremstyle{TheoremNum}
    \newtheorem{theoremnum}{Theorem}

\newcommand{\bm}[1]{\mathbf{#1}}

\newcommand{\envs}{\mathcal{E}}
\newcommand{\states}{\mathcal{X}}
\newcommand{\indep}{\perp}
\icmltitlerunning{Invariant Causal Prediction for Block MDPs}

\begin{document}

\twocolumn[

\icmltitle{Invariant Causal Prediction for Block MDPs}



\icmlsetsymbol{equal}{*}

\begin{icmlauthorlist}
\icmlauthor{Amy Zhang}{equal,mcgill,mila,fair}
\icmlauthor{Clare Lyle}{equal,ox}
\icmlauthor{Shagun Sodhani}{fair}
\icmlauthor{Angelos Filos}{ox}
\icmlauthor{Marta Kwiatkowska}{ox}
\icmlauthor{Joelle Pineau}{mcgill,mila,fair}
\icmlauthor{Yarin Gal}{ox}
\icmlauthor{Doina Precup}{mcgill,mila,deepmind}
\end{icmlauthorlist}

\icmlaffiliation{mcgill}{McGill University}
\icmlaffiliation{deepmind}{Deepmind}
\icmlaffiliation{mila}{Mila}
\icmlaffiliation{fair}{Facebook AI Research}
\icmlaffiliation{ox}{University of Oxford}

\icmlcorrespondingauthor{Amy Zhang}{amy.x.zhang@mail.mcgill.ca}
\icmlcorrespondingauthor{Clare Lyle}{clare.lyle@univ.ox.ac.uk}

\icmlkeywords{Machine Learning, ICML}

\vskip 0.3in
]



\printAffiliationsAndNotice{\icmlEqualContribution} 

\begin{abstract}
Generalization across environments is critical to the successful application of reinforcement learning algorithms to real-world challenges. In this paper, we consider the problem of learning abstractions that generalize in block MDPs, families of environments with a shared latent state space and dynamics structure over that latent space, but varying observations. We leverage tools from causal inference to propose a method of invariant prediction to learn \textit{model-irrelevance state abstractions} (MISA) that generalize to novel observations in the multi-environment setting. We prove that for certain classes of environments, this approach outputs with high probability a state abstraction corresponding to the causal feature set with respect to the return. We further provide more general bounds on model error and generalization error in the multi-environment setting, in the process showing a connection between causal variable selection and the state abstraction framework for MDPs. We give empirical evidence that our methods work in both linear and nonlinear settings, attaining improved generalization over single- and multi-task baselines.
\end{abstract}

\section{Introduction}

The canonical reinforcement learning (RL) problem assumes an agent interacting with a single MDP with a fixed observation space and dynamics structure. This assumption is difficult to ensure in practice, where state spaces are often large and infeasible to explore entirely during training. However, there is often a latent structure to be leveraged to allow for good generalization. 
As an example, a robot's sensors may be moved, or the lighting conditions in a room may change, but the physical dynamics of the environment are still the same. These are examples of environment-specific characteristics that current RL algorithms often overfit to.
In the worst case, some training environments may contain spurious correlations that will not be present at test time, causing catastrophic failures in generalization~\cite{azhang2018natrl,Song2020Observational}. To develop algorithms that will be robust to these sorts of changes, we must consider problem settings that allow for multiple environments with a shared dynamics structure. 

Recent prior works~\cite{amit2018mlpacbayes,yin2019meta} have developed generalization bounds for the multi-task problem, but they depend on the number of tasks seen at training time, which can be prohibitively expensive given how sample inefficient RL is even in the single task regime. To obtain stronger generalization results, we propose to consider a problem which we refer to as `multi-environment' RL: like multi-task RL, the agent seeks to maximize return on a set of environments, but only some of which can be trained on. 
We make the assumption that there exists some latent causal structure that is shared among all of the environments, and that the sources of variability between environments do not affect reward. This family of environments is called a \textit{Block MDP}~\cite{du2019pcid}, in which the observations may change, but the latent states, dynamics, and reward function are the same. A formal definition of this type of MDP will be presented in \cref{sec:problem_setup}. 

Though the setting we consider is a subset of the multi-task RL problem, we show in this work that the added assumption of shared structure allows for much stronger generalization results than have been obtained by multi-task approaches. Naive application of generalization bounds to the multi-task reinforcement learning setting is very loose because the learner is typically given access to only a few tasks relative to the number of samples from each task.
Indeed, \citet{cobbe2018genrl,czhang2018genrl} find that agents trained using standard methods require many thousands of environments before succeeding at `generalizing' to new environments. 

The main contribution of this paper is to use tools from \textit{causal inference} to address generalization in the Block MDP setting, proposing a new method based on the \textit{invariant causal prediction} literature. In certain linear function approximation settings, we demonstrate that this method will, with high probability, learn an optimal state abstraction that generalizes across all environments using many fewer training environments than would be necessary for standard PAC bounds. We replace this PAC requirement with requirements from causal inference on the \textit{types} of environments seen at training time. We then draw a connection between bisimulation and the minimal causal set of variables found by our algorithm, providing bounds on the model error and sample complexity of the method. We further show that using analogous invariant prediction methods for the nonlinear function approximation setting can yield improved generalization performance over multi-task and single-task baselines. We relate this method to previous work on learning representations of MDPs~\cite{gelada2019deepmdp,luo2018algorithmic} and develop multi-task generalization bounds for such representations. Code is available at \burl{https://github.com/facebookresearch/icp-block-mdp}.

\section{Background}

\subsection{State Abstractions and Bisimulation}
State abstractions have been studied as a way to distinguish relevant from irrelevant information~\cite{li2006stateabs} in order to create a more compact representation for easier decision making and planning. \citet{bertsekas1989bounds,Roy06sabounds} provide bounds for approximation errors for various aggregation methods, and \citet{li2006stateabs} discuss the merits of \textit{abstraction discovery} as a way to solve related MDPs.

Bisimulation relations are a type of state abstraction that offers a mathematically precise definition of what it means for two environments to `share the same structure'~\cite{larsen1989bisim,Givan2003EquivalenceNA}. We say that two states are bisimilar if they share the same expected reward and equivalent distributions over the next bisimilar states. 
For example, if a robot is given the task of washing the dishes in a kitchen, changing the wallpaper in the kitchen doesn't change anything relevant to the task. One then could define a bisimulation relation that equates observations based on the locations and soil levels of dishes in the room and ignores the wallpaper. These relations can be used to simplify the state space for tasks like policy transfer \citep{castro2010using}, and are intimately tied to state abstraction. For example, the \textit{model-irrelevance abstraction} described by \citet{li2006stateabs} is precisely characterized as a bisimulation relation.
\begin{definition}[Bisimulation Relations~\citep{Givan2003EquivalenceNA}]
Given an MDP $\mathcal{M}$, an equivalence relation $B$ between states is a bisimulation relation if for all states $s_1,s_2\in\mathcal{S}$ that are equivalent under $B$ (i.e. $s_1Bs_2$), the following conditions hold for all actions $a\in\mathcal{A}$:
\begin{equation*}
\begin{split}
    R(s_1,a)&=R(s_2,a) \\
    \mathcal{P}(G|s_1,a)&=\mathcal{P}(G|s_2,a),\forall G\in\mathcal{S}/B
\end{split}
\end{equation*}
Where $\mathcal{S}/B$ denotes the partition of $\mathcal{S}$ under the relation $B$, the set of all groups of equivalent states, and where $\mathcal{P}(G|s,a)=\sum_{s'\in G}\mathcal{P}(s'|s,a).$
\end{definition}
Whereas this definition was originally designed for the single MDP setting to find \textit{bisimilar} states within an MDP, we are now trying to find bisimilar states across different MDPs, or different experimental conditions. One can intuitively think of this carrying over by imagining all experimental conditions $i$ mapped to a single super-MDP with state space $\mathcal{S}=\cup_i\mathcal{S}_i$ where we give up the irreducibility assumption, i.e. we can no longer reach every state $s_i$ from any other state $s_j$.
Specifically, we say that two MDPs $M_1$ and $M_2$ are bisimilar if there exist bisimulation relations $B_1$ and $B_2$ such that $M_1/B_1$ is isomorphic to $M_2/B_2$. \textit{Bisimilar  MDPs} are therefore MDPs which are behaviourally the same.


\subsection{Causal Inference Using Invariant Prediction}
\citet{peters2016icp} first introduced an algorithm, Invariant Causal Prediction (ICP), to find the \textit{causal feature set}, the minimal set of features which are causal predictors of a target variable, by exploiting the fact that causal models have an invariance property~\cite{pearl2009do,scholkopf2012causal}. \citet{arjovsky2019irm} extend this work by proposing invariant risk minimization (IRM, see \cref{eq:irm}), augmenting empirical risk minimization to learn a data representation free of spurious correlations. They assume there exists some partition of the training data $\mathcal{X}$ into \textit{experiments} $e \in \mathcal{E}$, and that the model's predictions take the form $Y^e = \mathbf{w}^\top \bm{\phi}{(X^e)}$. IRM aims to learn a representation $\bm{\phi}$ for which the optimal linear classifier, $\mathbf{w}$, is invariant across $e$, where optimality is defined as minimizing the empirical risk $R^e$. We can then expect this representation and classifier to have low risk in new experiments $e$, which have the same causal structure as the training set.
\begin{equation}\label{eq:irm}
\begin{split}
    &\min_{{\begin{subarray}{l}\bm{\phi}: \mathcal{X} \rightarrow \mathbb{R}^d\\
    \mathbf{w} \in \mathbb{R}^d \end{subarray}}} \sum_{e \in \mathcal{E}} R^e(\mathbf{w}^\top \bm{\phi}(X^e)) \\
    &\text{ s.t. } \mathbf{w} \in \underset{\bar{\mathbf{w}} \in \mathbb{R}^d}{\text{arg min }} R^e(\bar{\mathbf{w}}^\top \bm{\phi}(X^e)) \quad \forall e \in \mathcal{E}.
\end{split}
\end{equation}

The IRM objective in \cref{eq:irm} can be thought of as a constrained optimization problem, where the objective is to learn a set of features $\phi$ for which the optimal classifier in each environment is the same. Conditioned on the environments corresponding to different interventions on the data-generating process, this is hypothesized to yield features that only depend on variables that bear a causal relationship to the predicted value. Because the constrained optimization problem is not generally feasible to optimize, \citet{arjovsky2019irm} propose a penalized optimization problem with a schedule on the penalty term as a tractable alternative.

\section{Problem Setup}
\label{sec:problem_setup}
We consider a family of environments $\mathcal{M}_\mathcal{E} = \{(\mathcal{X}_e, \mathcal{A}, \mathcal{R}_e, \mathcal{T}_e, \gamma) | \; e \in \envs\}$, where $\mathcal{E}$ is some index set. For simplicity of notation, we drop the subscript $e$ when referring to the union over all environments $\mathcal{E}$. Our goal is to use a subset $\envs_{\text{train}} \subset \envs$ of these environments to learn a representation $\phi:\states \rightarrow \mathbb{R}^d $ which enables generalization of a learned policy to \textit{every} environment. We denote the number of training environments as $N:=|\envs_{\text{train}}|$. We assume that the environments share some structure, and consider different degrees to which this structure may be shared.

\subsection{The Block MDP}
Block MDPs~\citep{du2019pcid} are described by a tuple $\langle \mathcal{S}, \mathcal{A}, \mathcal{X}, p, q, R \rangle$ with a finite, unobservable state space $\mathcal{S}$, finite action space $\mathcal{A}$, and possibly infinite, but observable space $\mathcal{X}$. $p$ denotes the latent transition distribution $p(s'|s,a)$ for $s,s'\in\mathcal{S}, a\in\mathcal{A}$, $q$ is the (possibly stochastic) emission function that gives the observations from the latent state $q(x|s)$ for $x\in\mathcal{X}, s\in\mathcal{S}$, and $R$ the reward function.  A graphical model of the interactions between the various variables can be found in \cref{fig:irm_model_irrelevant}.
\begin{assumption}[Block structure~\citep{du2019pcid}]
Each observation $x$ uniquely determines its generating state $s$. That is, the observation space $\mathcal{X}$ can be partitioned into disjoint blocks $\mathcal{X}_s$, each containing the support of the conditional distribution $q(\cdot|s)$.
\end{assumption}
This assumption gives us the Markov property in $\mathcal{X}$. 
We translate the block MDP to our multi-environment setting as follows. If a family of environments $\mathcal{M}_\mathcal{E}$ satisfies the block MDP assumption, then each $e \in \mathcal{E}$ corresponds to an emission function $q_e$, with $S,A, \mathcal{X}$ and $p$ shared for all $M \in \mathcal{M}_\mathcal{E}$. We will move the potential randomness from $q_e$ into an auxiliary variable $\eta \in \Omega$, with $\Omega$ some probability space, and write $q_e(\eta, s)$. Further, we require that if $\text{range}(q_e (\cdot, s)) \cap \text{range}(q_{e'}(\cdot, s')) \neq \emptyset$, then $s = s'$. 
The objective is to learn a useful state abstraction to promote generalization across the different emission functions $q_e$, given that only a subset is provided for training. \citet{Song2020Observational} also describes a similar POMDP setting where there is an additional observation function, but assume information can be lost. We note that this problem can be made arbitrarily difficult if each $q_e$ has a disjoint range, but will focus on settings where the $q_e$ overlap in structured ways -- for example, where $q_e$ is the concatenation of the noise and state variables: $q_e(\eta, s) = s \oplus f(\eta)$.

\subsection{Relaxations}
\begin{figure}
    \centering
    \includegraphics[width=0.8\linewidth]{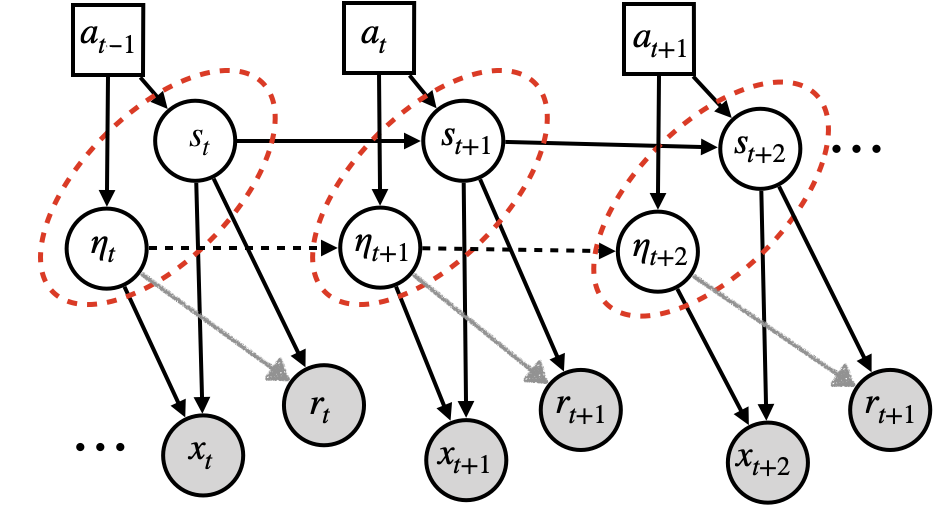}
    \caption{Graphical model of a block MDP with stochastic, correlated observations, with an IRM goal to extract $s$ from the sequence of observations, and discard the spurious noise $\eta$. Red dashed ovals indicate the entire tangled latent state at each timestep. Black dashed lines and grey lines indicate two additional tiers of difficulty to consider.}
    \label{fig:irm_model_irrelevant}
\end{figure}

\textbf{Spurious correlations.} Our initial presentation of the block MDP assumes that the noise variable $\eta$ is sampled randomly at every time step, which prevents multi-timestep correlations (\cref{fig:irm_model_irrelevant} in black, solid lines). We therefore also consider a more realistic \textit{relaxed block MDP}, where spurious variables may have different transition dynamics across the different environments so long as these correlations do not affect the expected reward (\cref{fig:irm_model_irrelevant},  now including black dashed lines). This is equivalent to augmenting each Block MDP in our family with a noise variable $\eta_e$, such that the observation $x = (q_e(\eta_e, s))$, and 
\begin{equation*}p(x'|x, a) = p(q^{-1}(x')|s, a) p_e(\eta_e'|s, \eta_e). \end{equation*}
We note that this section still satisfies Assumption 1.

\textbf{Realizability.} Though our analysis will require Assumption 1, we claim that this is a reasonable requirement as it makes the learning problem realizable. Relaxing Assumption 1 means that the value function learning problem may become ill-posed, as the same observation can map to entirely different states in the latent MDP with different values, making our environment partially observable (a POMDP, \cref{fig:irm_model_irrelevant} with grey lines). We provide a lower bound on the value approximation error attainable in this setting in the appendix (\cref{thm:nonrealizable}).

\subsection{Assumptions on causal structure}
State abstraction and causal inference both aim to eliminate spurious features in a learning algorithm's input. However, these two approaches are applied to drastically different types of problems. Though we demonstrate that causal inference methods can be applied to reinforcement learning, this will require some assumption on how causal mechanisms are observed. Definitions of the notation used in this section are deferred to the appendix, though they are standard in the causal inference literature. 

The key assumption we make is that the variables in the environment state at time $t$ can only affect the values of the state at time $t+1$, and can only affect the reward at time $t$. This assumption allows us to consider the state and action at time $t$ as the only candidate for causal parents of the state at time $t+1$ and of the reward at time $t$. This assumption is crucial to the Markov behaviour of the Markov decision process. We refer the reader to \cref{fig:statgm} to demonstrate how causal graphical models can be translated to this setting.

\begin{assumption}[Temporal Causal Mechanisms]\label{assmpt:causal_mechanisms}
Let $x^1$ and $x^2$ be components of the observation $x$. Then when no intervention is performed on the environment, we have the following independence.
\begin{equation}
     X^1_{t+1} \perp X^2_{t+1} | x_t
\end{equation} 
\end{assumption}
\begin{assumption}[Environment Interventions]\label{assmpt:envs}
Let $\mathcal{X} = X_1 \times \dots \times X_n$, and $\mathcal{S} = X_{i_1} \times \dots X_{i_k}$. Each environment $e \in \mathcal{E}$ corresponds to a do- \citep{pearl2009do} or soft \citep{eberhardt2007interventions} intervention on a single variable $x_i$ in the observation space. 
\end{assumption}

This assumption allows us to use tools from causal inference to identify candidate model-irrelevance state abstractions that may hold across an entire family of MDPs, rather than only the ones observed, based on using the state at one timestep to predict values at the next timestep. In the setting of Assumption 3, we can reconstruct the block MDP emission function $q$ by concatenating the spurious variables from $\mathcal{X} \setminus \mathcal{S}$ to $\mathcal{S}$. We discuss some constraints on interventions necessary to satisfy the block MDP assumption in the appendix.

\section{Connecting State Abstractions to Causal Feature Sets}

Invariant causal prediction aims to identify a set $S$ of causal variables such that a linear predictor with support on $S$ will attain consistent performance over all environments. In other words, ICP removes irrelevant variables from the input, just as state abstractions remove irrelevant information from the environment's observations. An attractive property of the block MDP setting is that it is easy to show that there does exist a model-irrelevance state abstraction $\phi$ for all MDPs in $\mathcal{M}_\mathcal{E}$ -- namely, the function mapping each observation $x$ to its generating latent state $\phi(x) = q^{-1}(x)$. The formalization and proof of this statement are deferred to the appendix (see \cref{thm:existence}). 

We consider whether, under Assumptions 1-3, such a state abstraction can be obtained by ICP. Intuitively, one would then expect that the \textit{causal variables} should have nice properties as a state abstraction. The following result confirms this to be the case; a state abstraction that selects the set of causal variables from the observation space of a block MDP will be a model-irrelevance abstraction for every environment $e \in \mathcal{E}$.

\begin{theorem} 
\label{thm:causalstate_modelirrelevance}
Consider a family of MDPs $M_\mathcal{E} = \{(\mathcal{X}, A, R, P_e, \gamma)|e \in \mathcal{E} \}$, with $\mathcal{X} = \mathbb{R}^k$. Let $M_\mathcal{E}$ satisfy Assumptions 1-3. Let $S_{R} \subseteq \{1, \dots, k\}$ be the set of variables such that the reward $R(x,a)$ is a function only of $[x]_{S_R}$ ($x$ restricted to the indices in $S_R$). Then let $S = \textbf{AN}(R)$ denote the ancestors of $S_R$ in the (fully observable) causal graph corresponding to the transition dynamics of $M_\mathcal{E}$. Then the state abstraction $\phi_S(x) = [x]_S$  is a \textit{model-irrelevance} abstraction for every $e \in \mathcal{E}$. 
\end{theorem}

\begin{figure}
    \centering
    \includegraphics[width=0.45\textwidth]{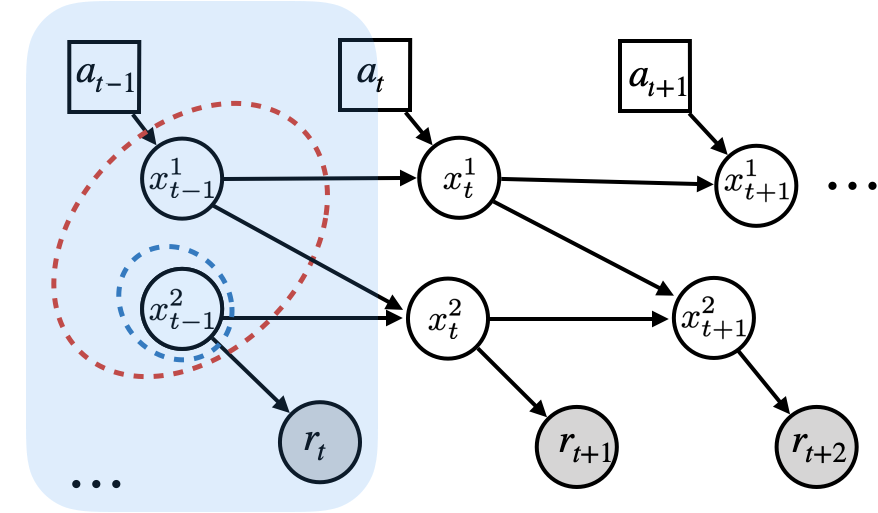}
    \caption{Graphical causal models with temporal dependence -- note that while $x^2$ (circled in blue) is the only causal parent of the reward, because its next-timestep distribution depends on $x^1$, a model-irrelevance state abstraction must include both variables. Shaded in blue: the graphical causal model of an MDP with states $s=(x^1, x^2)$ when ignoring timesteps.}
    \label{fig:statgm}
\end{figure}

An important detail in the previous result is the model irrelevance state abstraction incorporates not just the parents of the reward, but also its ancestors. This is because in RL, we seek to model \textit{return} rather than solely rewards, which requires a state abstraction that can capture multi-timestep interactions. We provide an illustration of this in \cref{fig:statgm}. As a concrete example, we note that in the popular benchmark CartPole, only position $x$ and angle $\theta$ are necessary to predict the reward. However, predicting the return requires $\dot{\theta}$ and $\dot{x}$, their respective velocities.

Learning a minimal $\phi$ in the setting of Theorem 1 using a single training environment may not always be possible. However, applying invariant causal prediction methods in the multi-environment setting will yield the minimal causal set of variables when the training environment interventions satisfy certain conditions necessary for the identifiability of the causal variables \citep{peters2016icp}.

\section{Block MDP Generalization Bounds}
We continue to relax the assumptions needed to learn a causal representation and look to the nonlinear setting. 
As a reminder, the goal of this work is to produce representations that will generalize from the training environments to a novel test environment. 
However, normal PAC generalization bounds require a much larger number of environments than one could expect to obtain in the reinforcement learning setting. The appeal of an invariant representation is that it may allow for theoretical guarantees on learning the right state abstraction with many fewer training environments, as discussed by \citet{peters2016icp}. If the learned state abstraction is close to capturing the true base MDP, then the model error in the test environment can be bounded by a function of the distance of the test environment's abstract state distribution to the training environments'. Though the requirements given in the following \cref{thm:model_error} are difficult to guarantee in practice, the result  will hold for any arbitrary learned state abstraction. 

\begin{theorem}[Model error bound]
\label{thm:model_error}
Consider an MDP $M$, with $M'$ denoting a coarser bisimulation of $M$. Let $\phi$ denote the mapping from states of $M$ to states of $M'$. 
Suppose that the dynamics of $M$ are $L$-Lipschitz w.r.t. $\phi(X)$ and that $T$ is some approximate transition model satisfying $\max_{s} \mathbb{E}\|T(\phi(s)) - \phi(T_M(s)) \| < \delta$, for some $\delta > 0$. Let $W_1(\pi_1, \pi_2)$ denote the 1-Wasserstein distance. Then
\begin{equation}
    \mathbb{E}_{x \sim M'}[\|T(\phi(x)) - \phi(T_{M'}(x)) \|] \leq \delta + 2LW_1(\pi_{\phi(M)}, \pi_{\phi(M')}).
\end{equation}
\end{theorem}
Proof found in \cref{app:proofs}.

Instead of assuming access to a bisimilar MDP $M'$, we can provide discrepancy bounds for an MDP $\bar{M}$ produced by a learned state representation $\phi(x)$, dynamics function $f_s$, and reward function $R$ using the distance in dynamics $J_D^\infty$ and reward $J_R^\infty$ of $\bar{M}$ to the underlying MDP $M$. We first define these distances,
\begin{equation}
\begin{split}
    J_R^\infty:=\sup_{x\in\mathcal{X},a\in\mathcal{A}}|R(\phi(x),a,\phi(x')) - r(x, a)| \\
    J_D^\infty:=\sup_{x\in\mathcal{X},a\in\mathcal{A}} W_1(f_s(\phi(x),a), \phi P(x,a)).
\end{split}
\end{equation}

\begin{theorem}
\label{thm:value_bounds}
Let $M$ be a block MDP and $\bar{M}$ the learned invariant MDP with a mapping $\phi:\mathcal{X}\mapsto \mathcal{Z}$. For any $L$-Lipschitz valued policy $\pi$ the value difference of that policy is bounded by
\begin{equation}
    |Q^\pi(x,a) - \bar{Q}^\pi (\phi(x),a)|\leq \frac{J_R^\infty + \gamma LJ_D^\infty}{1-\gamma},
\end{equation}
where $Q^\pi$ is the value function for $\pi$ in $M$ and $\bar{Q}^\pi$ is the value function for $\pi$ in $\bar{M}$.
\end{theorem}
Proof found in \cref{app:proofs}.
This gives us a bound on generalization performance that depends on the supremum of the dynamics and reward errors, which correspondingly is a regression problem that will depend on $\sum_{e\in\mathcal{E}}n_e$, the number of samples we have in aggregate over all training environments rather than the number of training environments, $|\mathcal{E}|$.  Recent generalization bounds for deep neural networks using Rademacher complexity~\cite{bartlett2017genbounds,arora2018genbounds} scale with a factor of $\frac{1}{\sqrt{n}}$ where $n$ is the number of samples. We can use $n:=\sum_{e\in\mathcal{E}}n_e$ for our setting, getting generalization bounds for the block MDP setting that scale with the number of samples in aggregate over all environments, an improvement over previous multi-task bounds that depend on $|\mathcal{E}|$.

\section{Methods}
Given these theoretical results, we propose two methods to learn invariant representations in the block MDP setting. Both methods take inspiration from invariant causal prediction, with the first being the direct application of linear ICP to select the causal variables in the state in the setting where variables are given. This corresponds to direct feature selection, which with high probability returns the minimal causal feature set. The second method is a gradient-based approach akin to the IRM objective, with no assumption of a linear causal relationship and a learned causal invariant representation. Like the IRM goal (\cref{eq:irm}), we aim to learn an invariant state abstraction from stochastic observations across different interventions $i$, and impose an additional invariance constraint.

\subsection{Variable Selection for Linear Predictors}
The following algorithm (\cref{alg:linear_misa}) returns a model-irrelevance state abstraction. We require the presence of a replay buffer $\mathcal{D}$, in which transitions are stored and tagged with the environment from which they came. The algorithm then applies ICP to find all causal ancestors of the reward iteratively. This approach has the benefit of inheriting many nice properties from ICP -- under suitable identifiability conditions, it will return the exact causal variable set to a specified degree of confidence. 

It also inherits inconvenient properties: the ICP algorithm is exponential in the number of variables, and so this method is not efficient for high-dimensional observation spaces. We are also restricted to considering linear relationships of the observation to the reward and next state. Further, because we take the union over iterative applications of ICP, the confidence parameter $\alpha$ used in each call must be adjusted accordingly. Given $n$ observation variables, we give a conservative bound of $\frac{\alpha}{n}$.

\begin{algorithm}[t]
\SetAlgoLined
\KwResult{$S \subset \{1, \dots, k\}$, the causal state variables}
\textbf{Input:} $\alpha$, a confidence parameter, $\mathcal{D}$, an replay buffer with observations $\mathcal{X}$. \; \\
 $S \gets \emptyset$\; \\
 stack $\gets$ r \\
 \While{stack is not empty}{
  $v$ = stack.pop() \\
  \If{$v \not \in S$}{
  $S' \gets$ \texttt{ICP}(v, $\mathcal{D}$, $\frac{\alpha}{\text{dim}(\mathcal{X})}$)  \\
  $S \gets S \; \cup S'$ \\
   stack.push($S'$)}
 
 }
\caption{Linear MISA: Model-irrelevance State Abstractions}
\label{alg:linear_misa}
\end{algorithm}

\subsection{Learning a Model-irrelevance State Abstraction}

\label{sec:model_irrelevant}

We design an objective to learn a dynamics preserving state abstraction $\mathcal{Z}$, or model-irrelevance abstraction~\citep{li2006stateabs}, where the similarity of the model is bounded by the model error in the environment setting shown in \cref{fig:irm_model_irrelevant}. This requires disentangling the state space into a minimal representation that causes reward $s_t:=\phi(x_t)$ and everything else $\eta_t:=\varphi(x_t)$. Our algorithm proceeds as follows. 

We assume the existence of an invariant state embedding, whose mapping function we denote by $\phi:\mathcal{X}\mapsto\mathcal{Z}$. We also assume an invariant dynamics model $f_s:\mathcal{A}\times \mathcal{Z}\mapsto\mathcal{Z}$, a task-specific dynamics model $f_\eta:\mathcal{A}\times \mathcal{H}\mapsto\mathcal{H}$, and an invariant reward model $r:\mathcal{Z}\times\mathcal{A}\times\mathcal{Z}\mapsto\mathbb{R}$ in the embedding space. To incorporate a meaningful objective and ground the learned representation, we need a decoder $\phi^{-1}:\mathcal{Z}\times\mathcal{H}\mapsto\mathcal{X}$. We assume $N>1$ training environments are given. 
The overall dynamics and reward objectives become
\begin{align*}
    J_D(\phi,\psi, f_s,f_\eta)=\sum_i \mathbb{E}_{\pi_{b_i}}\big[&(\phi^{-1} (f_s(a,\phi(x_i)), \\
    &f_\eta(a,\psi(x_i))) - x_i')^2\big],
\end{align*}
\begin{equation*}
    J_R(\phi,R)=\sum_i \mathbb{E}_{\pi_{b_i}}\big[(R(\phi(x_i), a,\phi(x_i')) - r_i')^2\big],
\end{equation*}
under data collected from behavioral policies $\pi_{b_i}$ for each experimental setting. 

Of course, this does not guarantee that the representation learned by $\phi$ will be minimal, so we incorporate additional regularization as an incentive. We train a task classifier on the shared latent representation $C:\mathcal{Z}\mapsto [0,1]^N$ with cross-entropy loss and employ an adversarial loss~\cite{tzeng2017ada} on $\phi$ to maximize the entropy of the classifier output to ensure task specific information is not passing through to $\mathcal{Z}$.

This gives us a final objective
\begin{equation}
\begin{split}
    &J_{\text{ALL}}(\phi, \psi, f_s, f_\eta, r) = \\ &J_D(\phi,\psi, f_s,f_\eta) + \alpha_R J_R(\phi,r) - \alpha_C H(C(\phi)),
\end{split}
\end{equation}
where $\alpha_R$ and $\alpha_C$ are hyperparameters and $H$ denotes entropy (\cref{alg:nonlinear_misa}). 

\begin{algorithm}[t]
\SetCustomAlgoRuledWidth{0.6\textwidth}
\SetAlgoLined
\KwResult{$\phi$, an invariant state encoder}
 $\pi \gets \pi_0$\; \\
 $\phi, f_s \gets \phi_0, f_{s,0}$ \; \\
 $\psi^e, f_\eta^e \gets \psi^e_0,f_{\eta,0}^e $ for $e \in \mathcal{E}$ \; \\
 $\mathcal{D}_e \gets \emptyset$ for $e \in \mathcal{E}$ \; \\
 \While{forever}
 {
    \For{$e \in \mathcal{E}$}
    {
        $a \gets \pi(x_e)$\; \\
        $x'_e, r \gets $ \texttt{step}$(x_e,a)$ \; \\
        \texttt{store}$(x_e,a,r,x'_e)$ \; \\
    }
    \For{$e \in \mathcal{E}$}
    {
    Sample batch $X_e$ from $\mathcal{D}_e$ \; \\
    $f_\eta^e,\psi^e \gets \nabla_{f_{\eta}^e,\psi^e} [J_D(X_e) ]$ \; \\
    }
    $f_s, \phi, r \gets \sum_{X_e}\nabla_{f_s,\phi} [J_{\text{ALL}}(X_e) ]$\; \\
    $C \gets \nabla_C$ \texttt{CE\_loss}$(C(\phi(\{x_e\}_{e\in\mathcal{E}}), \{e\}_{e\in\mathcal{E}})$ \; \\

 }
 \caption{Nonlinear Model-irrelevance State Abstraction (MISA) Learning}
 \label{alg:nonlinear_misa}
\end{algorithm}

\section{Results}
We evaluate both linear and non-linear versions of MISA, in corresponding Block MDP settings with both linear and non-linear dynamics. First, we examine model error in environments with low-dimensional (\cref{sec:model_linear}) and high-dimensional (\cref{sec:model_nonlinear}) observations and demonstrate the ability for MISA to zero-shot generalize to unseen environments. We next look to imitation learning in a rich observation setting (\cref{sec:imitation_learning}) and show non-linear MISA generalize to new camera angles. Finally, we explore end-to-end reinforcement learning in the low-dimensional observation setting with correlated noise (\cref{sec:reinforcement_learning}) and again show generalization capabilities where single task and multi-task methods fail.

\subsection{Model Learning}
\subsubsection{Linear Setting}
\label{sec:model_linear}
We first evaluate the linear MISA algorithm in \cref{alg:linear_misa}. To empirically evaluate whether eliminating spurious variables from a representation is necessary to guarantee generalization, we consider a simple family of MDPs with state space $\mathcal{X} = \{ (x_1, x_2, x_3)\}$, with a transition dynamics structure such that $x_1^{t+1} = x_1^t + \epsilon_1^e$, $x_2^{t+1} = x_2^t + \epsilon_2^e$, and $x_3^{t+1} = x_2^t + \epsilon_3^e$. We train on 3 environments with soft interventions on each noise variable. We then run the linear MISA algorithm on batch data from these 3 environments to get a state abstraction $\phi(x) = \{x_1, x_2\}$, and then train 2 linear predictors on $\phi(x)$ and $x$. We then evaluate the \textit{generalization performance} for novel environments that correspond to different hard interventions on the value of the $x_3$ variable. We observe that the predictor trained on $\phi(x)$ attains zero generalization error because it zeros out $x_3$ automatically. However, any nonzero weight on $x_3$ in the least-squares predictor will lead to arbitrarily large generalization error, which is precisely what we observe in Figure \ref{fig:icp_result}.

\begin{figure}[h]
    \centering
        \vspace{-10pt}
    \includegraphics[width=0.4\textwidth]{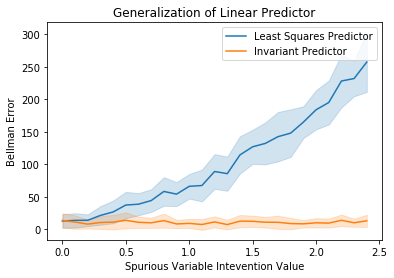}
        \vspace{-10pt}
    \caption{The presence of spurious \textit{uncorrelated} variables in the state can still lead to poor generalization of linear function approximation methods. Invariant Causal Prediction methods can eliminate these spurious variables altogether.}
    \label{fig:icp_result}
\end{figure}

\subsubsection{Rich Observation Setting}
\label{sec:model_nonlinear}
We next test the gradient-based MISA method (\cref{alg:nonlinear_misa}) in a setting with nonlinear  dynamics and rich observations. Instead of having access to observation variables and selecting the minimal causal feature set, we are tasked with learning the invariant causal representation. We randomly initialize the background color of two train environments from Deepmind Control~\citep{deepmindcontrolsuite2018} from range $[0, 255]$. We also randomly initialize another two backgrounds for evaluation. The orange line in \cref{fig:cheetah_run} shows performance on the evaluation environments in comparison to three baselines. In the first, we only train on a single environment and test on another with our method, (\texttt{MISA - 1 env}). Without more than a single experiment to observe at training time, there is no way to disentangle what is causal in terms of dynamics, and what is not. In the second baseline, we combine data from the two environments and train a model over all data (\texttt{Baseline - 1 decoder}). The third is another invariance-based method which uses a gradient penalty, IRM~\cite{arjovsky2019irm}. In the second case the error is tempered by seeing variance in the two environments at training time, but it is not as effective as MISA with two environments at disentangling what is invariant, and therefore causal with respect to dynamics, and what is not. With IRM, the loss starts much higher but very slowly decreases, and we find it is very brittle to tune in practice. Implementation details found in \cref{app:model_nonlinear_implementation}.

\begin{figure}[h]
    \centering
    \includegraphics[width=.8\linewidth]{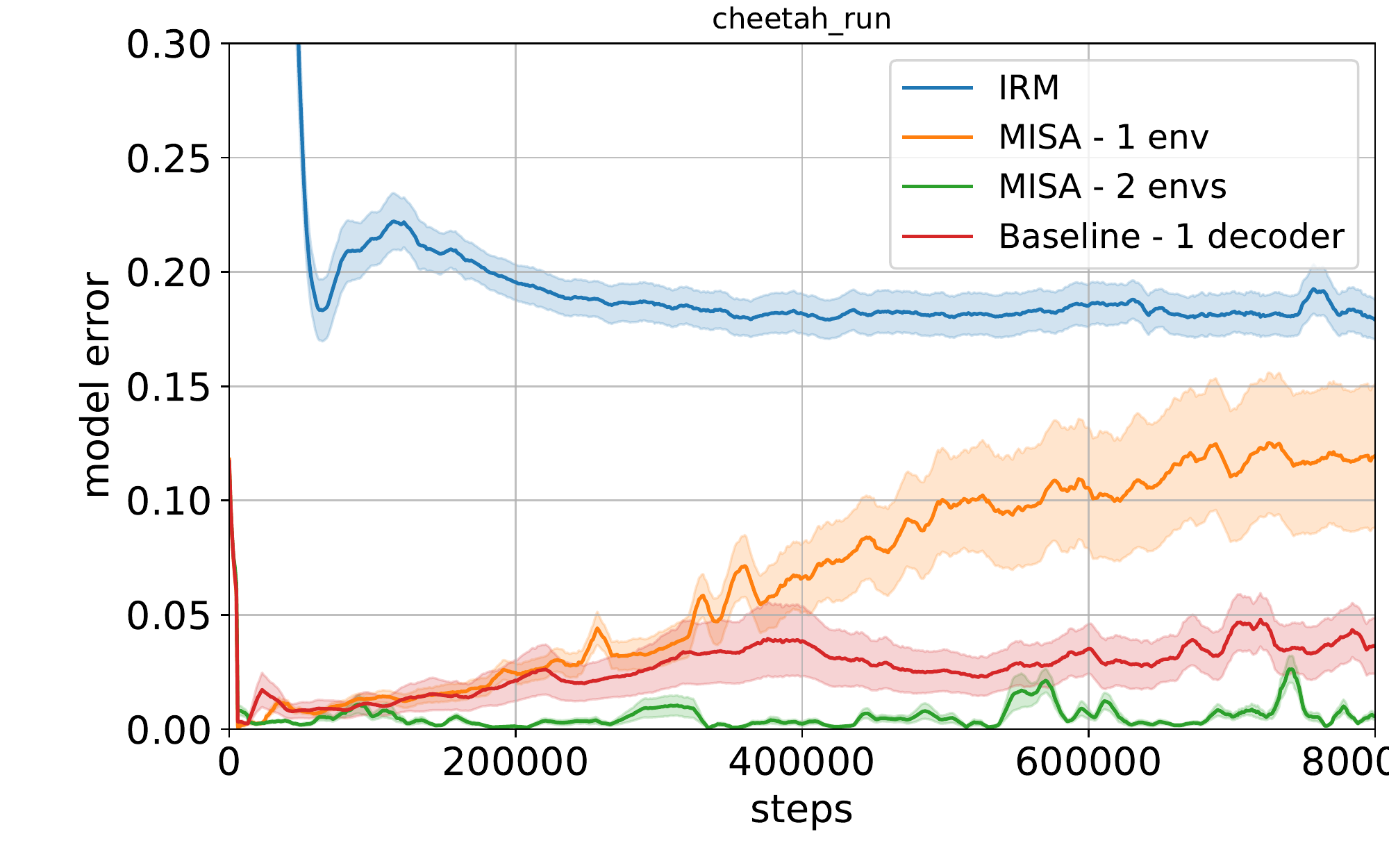}
    \vspace{-10pt}
    \caption{Model error on evaluation environments on Cheetah Run from Deepmind Control. 10 seeds, with one standard error shaded.}
    \label{fig:cheetah_run}
    \vspace{-5pt}
\end{figure}

\subsection{Imitation Learning}
\label{sec:imitation_learning}
In this setup, we first train an expert policy using the proprioceptive state of Cheetah Run from \cite{deepmindcontrolsuite2018}. We then use this policy to collect a dataset for imitation learning in each of two training environments. When rendering these low dimensional images, we alter the camera angles in the different environments (\cref{fig:imitation_learning_envs}). We report the generalization performance as the test error when predicting actions in \cref{fig:imitation_learning}. While we see test error does increase with our method, MISA, the error growth is significantly slower compared to single task and multi-task baselines. 

\begin{figure}[h]
    \centering
    \includegraphics[width=0.08\textwidth]{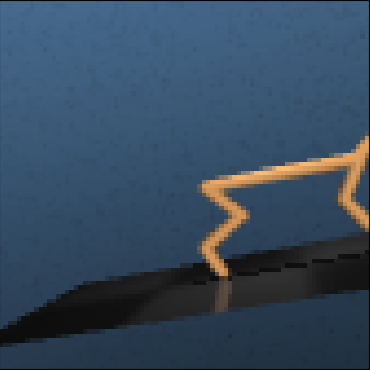} \hspace{20pt}
    \includegraphics[width=0.08\textwidth]{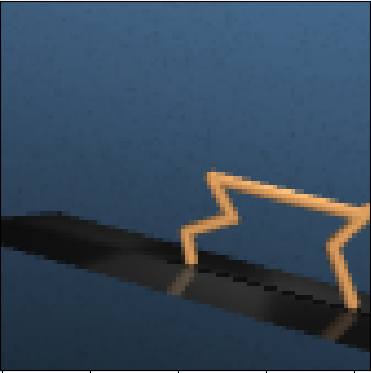} \hspace{20pt}
    \includegraphics[width=0.08\textwidth]{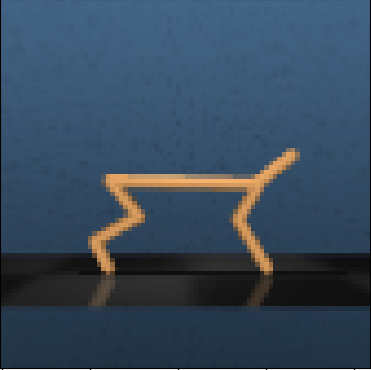}
        \vspace{-10pt}
    \caption{The Cheetah Run environment from Deepmind Control with different camera angles. The first two images are from the training environments and the last image is from evaluation environment.}
    \label{fig:imitation_learning_envs}
\end{figure}

\begin{figure}[h]
    \centering
    \vspace{-5pt}
    \includegraphics[width=0.4\textwidth]{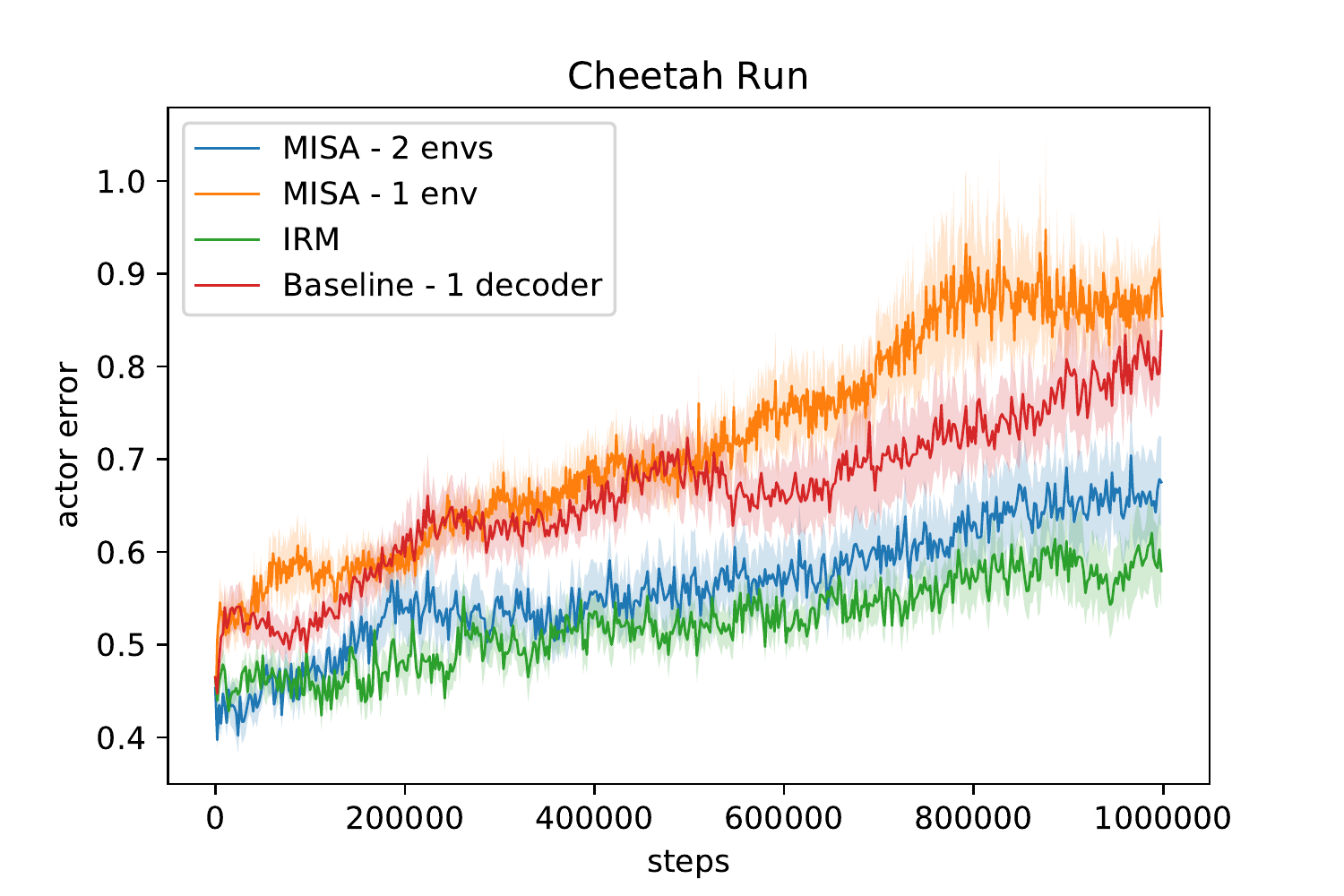}
        \vspace{-10pt}
    \caption{Actor error on evaluation environments on Cheetah Run from Deepmind Control. 10 seeds, with one standard error shaded.}
    \vspace{-5pt}
    \label{fig:imitation_learning}
\end{figure}

\subsection{Reinforcement Learning}
\label{sec:reinforcement_learning}
We go back to the proprioceptive state in the \texttt{cartpole\_swingup} environment in Deepmind Control~\cite{deepmindcontrolsuite2018} to show that we can learn MISA while training a policy. We use Soft Actor Critic~\cite{haarnoja2018sac} with an additional linear encoder, and add spurious correlated dimensions which are a multiplicative factor of the original state space. We also add an additional environment identifier to the observation. This multiplicative factor varies across environments, and we train on two environments with $1\times$ and $2\times$, and test on $3\times$.  Like \citet{arjovsky2019irm}, we also incorporate noise on the causal state to make the task harder, specifically Gaussian noise $\mathcal{N}(0, 0.01)$ to the true state dimension. This incentivizes the agent to attend to the spuriously correlated dimension instead, which has no noise. In \cref{fig:cartpole_swingup_rl} we see the generalization gap drastically improve with our method in comparison to training SAC with data over all environments in aggregate and with IRM~\cite{arjovsky2019irm} implemented on the critic loss. Implementation details and more information about Soft Actor Critic can be found in \cref{app:rl_implementation}.

\begin{figure}[h]
    \centering
    \includegraphics[width=0.4\textwidth]{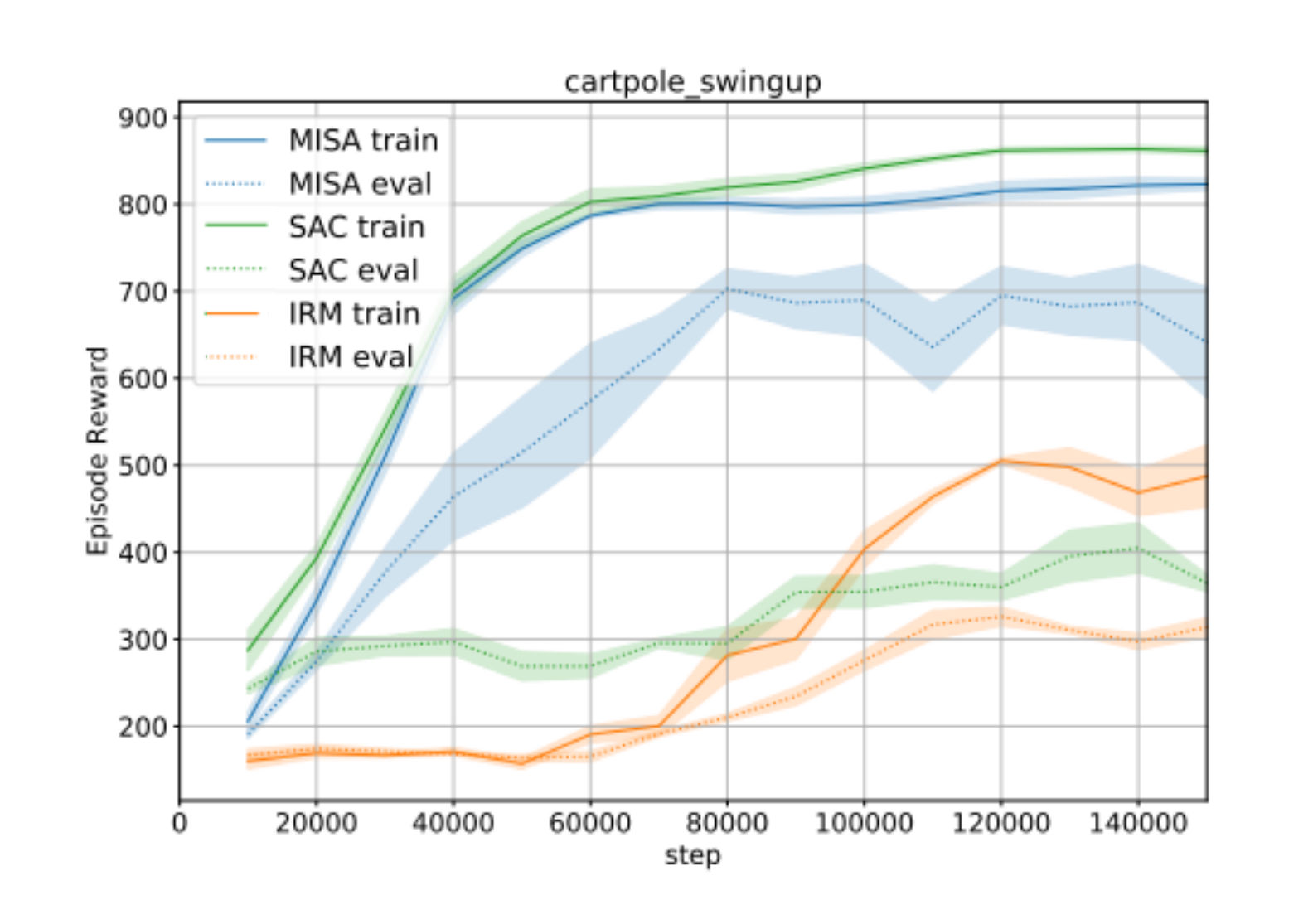}
    \caption{Generalization gap in SAC performance with 2 training environments on \texttt{cartpole\_swingup} from DMC. Evaluated with 10 seeds, standard error shaded.}
    \label{fig:cartpole_swingup_rl}
\end{figure}
\section{Related Work}

\subsection{Prior Work on Generalization Bounds}
Generalization bounds provide guarantees on the test set error attained by an algorithm. Most of these bounds are probabilistic and targeted at the supervised setting, falling into the PAC (Probably Approximately Correct) framework.
PAC bounds give probabilistic guarantees on a model's true error as a function of its train set error and the complexity of the function class encoded by the model. Many measures of hypothesis class complexity exist: the Vapnik-Chernovenkis (VC) dimension \cite{vapnik71uniform}, the Lipschitz constant, and classification margin of a neural network \cite{bartlett2017spectrally}, and second-order properties of the loss landscape \cite{neyshabur2018the} are just a few of many.

Analogous techniques can be applied to Bayesian methods, giving rise to PAC-Bayes bounds~\cite{mcallester1999pac}. This family of bounds can be optimized to yield non-vacuous bounds on the test error of over-parametrized neural networks \cite{dziugaitecomputing}, and have demonstrated strong empirical correlation with model generalization \citep{jiang*2020fantastic}. More recently, \citet{amit2018mlpacbayes,yin2019meta} introduce a PAC-Bayes bound for the multi-task setting dependent on the number of tasks seen at training time. 

\citet{strehl2006pac} extend PAC framework to reinforcement learning, defining a new class of bounds called PAC-MDP. An algorithm is PAC-MDP if for any $\epsilon$ and $\delta$, the sample complexity of the algorithm is less than some polynomial in $(S,A,1/\epsilon,1/\delta,1/(1-\gamma))$ with probability at least $1-\delta$. The authors provide a PAC-MDP algorithm for model-free Q-learning.
\citet{lattimore2012pac} offers lower and upper bounds on the sample complexity of learning near-optimal behavior in MDPs by modifying the Upper Confidence RL (UCRL) algorithm~\cite{jaksch2010ucrl}.

\subsection{Multi-Task Reinforcement Learning}
\citet{teh2017distral,borsa2016mtrl} handle multi-task reinforcement learning with a shared ``distilled" policy~\cite{teh2017distral} and shared state-action representation~\cite{borsa2016mtrl} to capture common or invariant behavior across all tasks. No assumptions are made about how these tasks relate to each other other than a shared state and action space. 

\citet{D'Eramo2020Sharing} show the benefits of learning a shared representation in multi-task settings with an approximate value iteration bound and \citet{brunskill2013mtrl} also demonstrate a PAC-MDP algorithm with improved sample efficiency bounds through transfer across similar tasks. Again, none of these works look to the multi-environment setting to explicitly leverage environment structure.
\citet{barreto2017successorfeatures} exploit successor features for transfer, making the assumption that the dynamics across tasks are the same, but the reward changes. However, they do not handle the setting where states are latent, and observations change.

\section{Discussion}
This work has demonstrated that given certain assumptions, we can use causal inference methods in reinforcement learning to learn an invariant causal representation that generalizes across environments with a shared causal structure. We have provided a framework for defining families of environments, and methods, for both the low dimensional linear value function approximation setting and the deep RL setting, which leverage invariant prediction to extract a causal representation of the state. We have further provided error bounds and identifiability results for these representations. We see this paper as a first step towards the more significant problem of learning useful representations for generalization across a broader class of environments. Some examples of potential applications include third-person imitation learning, sim2real transfer, and, related to sim2real transfer, the use of privileged information in one task (the simulation) as grounding and generalization to new observation spaces ~\cite{salter2019attention}.

\section{Acknowledgements}
MK has received funding from the European Research Council (ERC)
under the European Union’s Horizon 2020 research and innovation programme
(grant agreement No.~834115). The authors would also like to thank Marlos Machado for helpful feedback in the writing process.
\bibliographystyle{apalike}
\bibliography{ref}

\newpage
\appendix
\onecolumn
\section{Notation}
We provide a summary of key notation used throughout the paper here.
\begin{align*}
    \textbf{PA}_\mathcal{G}(X) &: \text{ the parents of node $X$ in the causal graph $\mathcal{G}$. When $\mathcal{G}$ is clear from the setting, abbreviate this notation to $\textbf{PA}(X)$. }\\
    \textbf{AN}_\mathcal{G}(X)&: \text{ the ancestors of node $X$ in $\mathcal{G}$ (again, $\mathcal{G}$ omitted when unambiguous).}\\
    [x]_S &: [x_{i_1}, \dots, x_{i_k} | i_j \in S] \\
    \pi_M &: \text{ the stationary distribution given by some fixed policy in an MDP $M$.}\\
    q &: \text{ the emission function of a block MDP.} \\
    \mathcal{E} &: \text{ a set of environments.}
\end{align*}
\section{Proofs}

\textbf{Technical notes and assumptions.} In order for the block MDP assumption to be satisfied, we will require that the interventions defining each environment only occur outside of the causal ancestors of the reward. Otherwise, the different environments will have different latent state dynamics, which violates our assumption that the environments are obtained by an noisy emission function from the latent state space $\mathcal{S}$. Although ICP will still find the correct causal variables in this setting, this state abstraction will no longer be a model irrelevance state abstraction over the union of the environments. 

\label{app:proofs}
\begin{theoremnum}[\ref{thm:causalstate_modelirrelevance}]
Consider a family of MDPs $M_\mathcal{E} = \{(\mathcal{X}, A, R, P_e, \gamma)|e \in \mathcal{E} \}$, with $\mathcal{X} = \mathbb{R}^k$. Let $M_\mathcal{E}$ satisfy Assumptions 1-3. Let $S_{R} \subseteq \{1, \dots, k\}$ be the set of variables such that the reward $R(x,a)$ is a function only of $[x]_{S_R}$ ($x$ restricted to the indices in $S_R$). Then let $S = \textbf{AN}(R)$ denote the ancestors of $S_R$ in the (fully observable) causal graph corresponding to the transition dynamics of $M_\mathcal{E}$. Then the state abstraction $\phi_S(x) = [x]_S$  is a \textit{model-irrelevance} abstraction for every $e \in \mathcal{E}$. 
\end{theoremnum}

\begin{proof}
To prove that $\phi_S$ is a model-irrelevance abstraction, we must first show that $r(x) = r(x')$ for any $x, x': \phi_S(x) = \phi_S(x')$. For this, we note that $\mathbb{E}[R(x)] = \int_{r \in \mathbb{R}} r dp(r|x) = \int_{r \in \mathbb{R}} r dp(r|[x]_S, [x]_{S^C})$ and, because by definition $S^C \subset \textbf{PA}(R)^C$, we have that $R \indep [x]_{S^C}$. Therefore, 
\begin{equation}
    \mathbb{E}[R(x)] = \int_{r \in \mathbb{R}} r dp(r|[x]_S) = \int_{r \in \mathbb{R}} r dp(r|[x']_S) = \mathbb{E}[R(x')].
\end{equation}

To show that $[x]_S$ is a MISA, we must also show that for any $x_1, x_2 $ such that $\phi(x_1) = \phi(x_2)$, and for any $e \in \envs$, the distribution over next state equivalence classes will be equal for $x_1$ and $x_2$.
$$\sum_{x' \in \phi^{-1}(\bar{X})} P^e_{x_1x'} = \sum_{x' \in \phi^{-1}(\bar{X})} P^e_{x_2x'}.$$  

For this, it suffices to observe that $S$ is closed under taking parents in the causal graph, and that by construction environments only contain interventions on variables outside of the causal set.  Specifically, we observe that the probability of seeing any particular equivalence class $[x']_S$ after state $x$ is only a function of $[x]_S$.
\begin{align*}
    P([x']_S | x) &= f([x]_S, [x']_S)
    \intertext{This allows us to define a natural decomposition of the transition function as follows.}
    P(x' | x) &= P\bigg ( [x]_S \oplus [x]_{S^C} \bigg| [x']_S \oplus [x']_{S^C} \bigg) \text{ which by the independent noise assumption gives}\\
    P(x'|x) &=f([x']_S, [x]_S) P([x']_{S^c}|x)
\end{align*} 
We further observe that since the components of $x$ are independent, $\sum_{[x']_{S^C}} P([x']_{S^C}|x) = 1$.
We now return to the property we want to show:
\begin{align*}
    \sum_{x' \in \phi^{-1}(\bar{x})} P^e_{x_1x'} &= \sum_{x' \in \phi^{-1}(\bar{x})} f([x_1]_S, [x']_S) P(x' | x_1 ) \\
    &= f(\phi(x_1), \bar{x})\sum_{[x']_{S^C}} P\bigg([x']_{S^C} \bigg | x_1 \bigg)\\
    &= f(\phi(x_1), \bar{x}) \\
\intertext{and because $\phi(x_1) = \phi(x_2)$, we have}
    &= f(\phi(x_2), \bar{x}) \\
\intertext{for which we can apply the previous chain of equalities backward to obtain}
&= \sum_{x' \in \phi^{-1}(\bar{x})} P^e_{x_2x'}  \\
\end{align*}

\end{proof}

\begin{proposition}[Identifiability and Uniqueness of Causal State Abstraction]\label{prop:identifiability}
In the setting of the previous theorem, assume the transition dynamics and reward are linear functions of the current state. If the training environment set $\mathcal{E}_{\text{train}}$ satisfies any of the conditions of Theorem 2 \citep{peters2016icp} with respect to each variable in \textbf{AN}(R), then the causal feature set $\phi_S$ is identifiable. Conversely, if the training environments don't contain sufficient interventions, then it may be that there exists a $\phi$ such that $\phi$ is a model irrelevance abstraction over $\mathcal{E}_{\text{train}}$, but not over $\mathcal{E}$ globally.
\end{proposition}

\begin{proof}
The proof of the first statement follows immediately from the iterative application of the identifiability result of \citet{peters2016icp} to each variable in the causal variables set.

For the converse, we consider a simple counterexample in which one variable $x_m$ is constant in every training environment, with value $v_m$. Then letting $S = \textbf{AN}(R)$, we observe that $S \cup \{m\}$ is also a model-irrelevance state abstraction.
First, we show $r(x_1)=r(x_2)$ for any $x_1,x_2:\phi_{S\cup\{m\}}(x_1)=\phi_{S\cup\{m\}}(x_2)$. 
\begin{equation*}
\begin{split}
    p(R|x_1,a)&=p(R|x_1|_S,a) \\
    &=p(R|x_1|_{S\cup\{m\}},a,m=v_m) \\
    &=p(R|(x_2|_{S\cup\{m\}},a,m=v_m) \\
    &=p(R|x_2,a)
\end{split}
\end{equation*}

Finally, we must show that $$\sum_{x' \in \phi_{S\cup\{m\}}^{-1}(\bar{X})} P_{x_1x'} = \sum_{x' \in \phi_{S\cup\{m\}}^{-1}(\bar{X})} P_{x_2x'}.$$
Again starting from the result of \cref{thm:causalstate_modelirrelevance} we have:
\begin{align*}
    \sum_{x' \in \phi_{S\cup\{m\}}^{-1}(\bar{x})} P_{x_1x'} &= \sum_{x' \in \phi_{S\cup\{m\}}^{-1}(\bar{x})} f(x_1|_{S\cup\{m\}},x'|_{S\cup\{m\}}) p(x'| x_1|_{(S\cup\{m\})^C}, m=v_m) \\
    &= f(\phi_{S\cup\{m\}}(x_1), \bar{x})\sum_{x' \in \phi_{S\cup\{m\}}^{-1}(\bar{x})} p(x' | x_1,m=v_m)\\
    &= f(\phi_{S\cup\{m\}}(x_1), \bar{x}) \\
\intertext{and because $\phi_{S\cup\{m\}}(x_1) = \phi_{S\cup\{m\}}(x_2)$, we have}
    &= f(\phi_{S\cup\{m\}}(x_2), \bar{x}) \\
\intertext{for which we can apply the previous chain of equalities backward to obtain}
&= \sum_{x' \in \phi_{S\cup\{m\}}^{-1}(\bar{x})} P_{x_2x'}  \\
\end{align*}

However, if one of the test environments contains the intervention $x_m \gets v_m + \mathcal{N}(0,\sigma^2)$, then the distribution over next-states in the new environment will violate the conditions for a model-irrelevance abstraction.  
\end{proof}

\begin{theoremnum}[\ref{thm:model_error}]
Consider an MDP $M$, with $M'$ denoting a coarser bisimulation of $M$. Let $\phi$ denote the mapping from states of $M$ to states of $M'$. 
Suppose that the dynamics of $M$ are $L$-Lipschitz w.r.t. $\phi(X)$ and that $T$ is some approximate transition model satisfying $\max_{s} \mathbb{E}\|T(\phi(s)) - \phi(T_M(s)) \| < \delta$, for some $\delta > 0$. Let $W_1(\pi_1, \pi_2)$ denote the 1-Wasserstein distance. Then
\begin{equation}
    \mathbb{E}_{x \sim M'}[\|T(\phi(x)) - \phi(T_{M'}(x)) \|] \leq \delta + 2LW_1(\pi_{\phi(M)}, \pi_{\phi(M')}).
\end{equation}
\end{theoremnum}
We will use the shorthand $\pi$ for $\pi_{\phi(M)}$, the distribution of state embeddings $\phi(M)$ corresponding to the behaviour policy, and $\pi'$ for $\pi_{\phi(M')}$ for the distribution of state embeddings $\phi(M')$ given by the behaviour policy.
\begin{proof}
\begin{align*}
    \mathbb{E}_{x \sim M'} [\|T(\phi(x)) - \phi(T_{M'}(x)) \|] &= \mathbb{E}_{x \sim M'} [\min_{y \in X_M} \|T(\phi(x)) - T(\phi(y)) + T(\phi(y)) - \phi(T_M(x)) \|] \\
    &\leq \mathbb{E}_{x \sim M'} [\min_{y \in X_M} \|T(\phi(x)) - T(\phi(y))\|
    \\ &+ \|T(\phi(y)) - \phi(T_M(y))\|+\| \phi(T_M(y)) - \phi(T_M(x))\|]\\
    \intertext{Let $\gamma$ be a coupling over the distributions of $\phi(M')$ and $\phi(M)$ such that $\mathbb{E}_{\gamma(\phi(x),\phi(y))} \|\phi(x)-\phi(y)\|= W_1(\pi, \pi')$}
    &\leq \mathbb{E}_{x \sim M'} [\mathbb{E}_{\gamma(\phi(y)|\phi(x))} \|T(\phi(x)) - T(\phi(y))\|] + \delta + L\|x-y\|]\\
    &\leq \mathbb{E}_{x \sim M'} [\mathbb{E}_{\gamma(\phi(y)|\phi(x))} L\|\phi(x) - \phi(y)\|+ \delta + L\|\phi(x)-\phi(y)\| ] \\
    &=  \mathbb{E}_{\gamma(\phi(x),\phi(y))}[L\|\phi(x) - \phi(y)\|+ \delta + L\|\phi(x)-\phi(y)\|] \\
    &= 2LW_1(\pi, \pi') + \delta
\end{align*}
\end{proof}

\begin{theorem}[Existence of model-irrelevance state abstractions]\label{thm:existence}
Let $\mathcal{E}$ denote some family of bisimilar MDPs with joint state space $\mathcal{X}_\mathcal{E} = \cup_{e \in \mathcal{E}} X_e$. Let the mapping from states in $M_e$ to the underlying abstract MDP $\bar{M}$ be denoted by $f_e$. Then if the states in $X_{\mathcal{E}}$ satisfy $x \in X_{e'} \cap X_e \implies f_e (x) = f_{e'}(x)$, then $\phi = \cup f_e$ is a model-irrelevance state abstraction for $\mathcal{E}$.
\end{theorem}
\begin{proof}
First, note that $\cup f_e$ is well-defined (because each $f$ agrees with the rest on the value of all states appearing in multiple tasks). Then $\phi$ will be a model-irrelevance abstraction for every MDP $M_e$ because it agrees with $f_e$ (a model-irrelevance abstraction).
\end{proof}
 \begin{theoremnum}[\ref{thm:value_bounds}]
    Let $M$ be our block MDP and $\bar{M}$ the learned invariant MDP with a mapping $\phi:\mathcal{X}\mapsto \mathcal{Z}$. For any $L$-Lipschitz valued policy $\pi$ the value difference is bounded by
    \begin{equation}
        |Q^\pi(x,a) - \bar{Q}^\pi (\phi(x),a)|\leq \frac{J_R^\infty + \gamma LJ_D^\infty}{1-\gamma}.
    \end{equation}
    \end{theoremnum}
\begin{proof}
    \begin{equation*}
    \begin{split}
        &\sup_{x_t\in\mathcal{X},a_t\in\mathcal{A}}|Q^\pi(x_t,a_t) - \bar{Q}^\pi (\phi(x_t),a_t)| \\
        &\qquad\leq \sup_{x_t\in\mathcal{X},a_t\in\mathcal{A}}|R(\phi(x_t),a,\phi(x_{t+1})) - r(x, a)| + \gamma \sup_{x_t\in\mathcal{X},a_t\in\mathcal{A}}|\mathbb{E}_{x_{t+1}\sim P(\cdot|x_t,a_t)}V^\pi(x_{t+1}) - \mathbb{E}_{z_{t+1}\sim f(\cdot|\phi(x_t),a_t)}\bar{V}^\pi(z_{t+1})| \\
        &\qquad= J_R^\infty + \gamma \sup_{x_t\in\mathcal{X},a_t\in\mathcal{A}}\big|\mathbb{E}_{x_{t+1}\sim P(\cdot|x_t,a_t)}[V^\pi(x_{t+1}) - \bar{V}^\pi(\phi(x_{t+1}))] + \mathbb{E}_{\substack{x_{t+1}\sim P(\cdot|x_t,a_t) \\ z_{t+1}\sim f(\cdot|\phi(x_t),a_t)}}[\bar{V}^\pi(\phi(x_{t+1})) - \bar{V}^\pi(z_{t+1})]\big| \\
        &\qquad\leq J_R^\infty + \gamma \sup_{x_t\in\mathcal{X},a_t\in\mathcal{A}}\big|\mathbb{E}_{x_{t+1}\sim P(\cdot|x_t,a_t)}[V^\pi(x_{t+1}) - \bar{V}^\pi(\phi(x_{t+1}))]\big|  \\
        &\qquad \phantom{\leq J_R^\infty~} + \gamma \sup_{x_t\in\mathcal{X},a_t\in\mathcal{A}} \big| \mathbb{E}_{\substack{x_{t+1}\sim P(\cdot|x_t,a_t) \\ z_{t+1}\sim f(\cdot|\phi(x_t),a_t)}}[\bar{V}^\pi(\phi(x_{t+1})) - \bar{V}^\pi(z_{t+1})]\big| \\
        &\qquad\leq J_R^\infty + \gamma \sup_{x_t\in\mathcal{X},a_t\in\mathcal{A}}\big|\mathbb{E}_{x_{t+1}\sim P(\cdot|x_t,a_t)}[V^\pi(x_{t+1}) - \bar{V}^\pi(\phi(x_{t+1}))]\big| + \gamma L \sup_{x_t\in\mathcal{X},a_t\in\mathcal{A}} W(\phi(P(\cdot|x_t,a_t)), f(\cdot|\phi(x_t),a_t)) \\
        &\qquad= J_R^\infty + \gamma \sup_{x_t\in\mathcal{X},a_t\in\mathcal{A}}\big|\mathbb{E}_{x_{t+1}\sim P(\cdot|x_t,a_t)}[V^\pi(x_{t+1}) - \bar{V}^\pi(\phi(x_{t+1}))]\big| + \gamma L J_D^\infty\\
        &\qquad\leq J_R^\infty + \gamma \sup_{x_t\in\mathcal{X},a_t\in\mathcal{A}}\mathbb{E}_{x_{t+1}\sim P(\cdot|x_t,a_t)}\big|[V^\pi(x_{t+1}) - \bar{V}^\pi(\phi(x_{t+1}))]\big| + \gamma L J_D^\infty\\
        &\qquad\leq J_R^\infty + \gamma \sup_{x_t\in\mathcal{X},a_t\in\mathcal{A}}\big|[V^\pi(x_t) - \bar{V}^\pi(\phi(x_t))]\big| + \gamma L J_D^\infty\\
        &\qquad\leq J_R^\infty + \gamma \sup_{x_t\in\mathcal{X},a_t\in\mathcal{A}}\big|[Q^\pi(x_{t-1},a_{t-1}) - \bar{Q}^\pi(\phi(x_{t-1}),a_{t-1})]\big| + \gamma L J_D^\infty\\
        &\qquad= \frac{J_R^\infty + \gamma LJ_D^\infty}{1-\gamma}
    \end{split}
    \end{equation*}
    \end{proof}

\begin{proposition}[Lower bound on abstraction error]\label{thm:nonrealizable}
Let $f_e$ be a mapping from $\mathcal{S} \rightarrow \mathcal{X}$. Fix some arbitrary policy $\rho$ and let $v(s)$ denote the value of state $s$ under $\rho$, with $\pi$ its stationary distribution. If $\exists $ $e, e', s, s'$ such that $f_e(s) = f_{e'}(s')$ (i.e. different states induce the same observation), then the following bound is a lower bound on the error obtained by a joint state abstraction over all environments.

\begin{equation}
\min_{\hat{v}}\frac{1}{|\mathcal{E}|} \sum_{e \in \mathcal{E}} \text{err}(\phi(X_e), \hat{v}) \geq \min_{s, s':v(s) \neq v(s')}
\bigg ( |v(s) - v(s')| \bigg ) P_\mathcal{E} \bigg ( (\phi(x) \neq f_e^{-1}(x) \bigg) \ge \delta \frac{H(V(S)|X) - 1}{\log |V(S)|}
\end{equation}
Where
\begin{equation*}
    \text{err}(\phi(X_e), \hat{v}) := \mathbb{E}_{\pi(X_e)} |\hat{v}(\phi(x)) - v(f_e^{-1}(x))|
\end{equation*}
and
\begin{equation*}
   \delta =  \min_{s, s':v(s) \neq v(s')}
\bigg ( |v(s) - v(s')| \bigg )
\end{equation*}
\end{proposition}

\begin{proof}
(Sketch) The error obtained by state abstraction will be at least the decoding error of values from abstract states scaled by $\delta$. This in turn depends on how effectively it is possible to decode a potentially lossy mapping from observations back to states. This leads to the second inequality, due to Fano, where the entropy $H(V(S)|X)$ is given by marginalizatiion with respect to $v(s)$ of the following probability distributions.

\begin{align*}
p(x) &= \frac{1}{|\mathcal{E}|} \sum_{s,e} \mathbbm{1}[f_e(s) = x] \pi(s) \\
p(s|x) &=\frac{1}{p(x)} \frac{1}{|\mathcal{E}|} \sum_{e}  \pi(s)
\end{align*}
\end{proof}

\section{Implementation Details}
\subsection{Model Learning: Rich Observations}
\label{app:model_nonlinear_implementation}
For the model learning experiments we use an almost identical encoder architecture as in~\citet{deepmindcontrolsuite2018}, with two more convolutional layers to the convnet trunk. Secondly, we use \texttt{ReLU} activations after each convolutional layer, instead of \texttt{ELU}. We use kernels of size $3 \times 3$ with $32$ channels  for all the convolutional layers and set stride to $1$ everywhere, except of the first convolutional layer, which has stride $2$. We then take the output of the convolutional net and feed it into a single fully-connected layer normalized by \texttt{LayerNorm}~\citep{ba2016layernorm}. Finally, we add \texttt{tanh} nonlinearity to the $50$ dimensional output of the fully-connected layer.

The decoder consists of one fully-connected layer that is then followed by four deconvolutional layers. We use \texttt{ReLU} activations after each layer, except the final deconvolutional layer that produces pixels representation. Each deconvolutional layer has kernels of size $3 \times 3$ with $32$ channels and stride $1$, except of the last layer, where stride is $2$.

The dynamics and reward models are all MLPs with two hidden layers with 200 neurons each and \texttt{ReLU} activations.

\subsection{Reinforcement Learning}
\label{app:rl_implementation}
For the reinforcement learning experiments we modify the Soft Actor-Critic PyTorch implementation by \citet{pytorch_sac} and augment with a shared encoder between the actor and critic, the general model $f_s$ and task-specific models $f_{\eta}^e$. The forward models are multi-layer perceptions with ReLU non-linearities and two hidden layers of 200 neurons each. The encoder is a linear layer that maps to a 50-dim hidden representation. We also use L1 regularization on the $S$ latent representation. We add two additional dimensions to the state space, a spurious correlation dimension that is a multiplicative factor of the last dimension of the ground truth state, as well as an environment id. We add Gaussian noise $\mathcal{N}(0, 0.01)$ to the original state dimension, similar to how \citet{arjovsky2019irm} incorporate noise in the label to make the task harder for the baseline.

Soft Actor Critic (SAC)~\cite{haarnoja2018sac} is an off-policy actor-critic method that uses the maximum entropy framework to derive soft policy iteration. At each iteration, SAC performs soft policy evaluation and improvement steps. The policy evaluation step fits a parametric soft Q-function $Q(x_t, a_t)$  using transitions sampled from the replay buffer $\mathcal{D}$ by minimizing the soft Bellman residual,
\begin{equation*}
    J(Q) = \mathbb{E}_{(x_t, x_t, r_t, x_{t+1}) \sim \mathcal{D}} \bigg[ \bigg(Q(x_t, a_t) - r_t - \gamma \Bar{V}(x_{t+1})\bigg)^2  \bigg].
\end{equation*}
The target value function $\Bar{V}$ is approximated via a Monte-Carlo estimate of the following expectation,
\begin{equation*}
    \Bar{V}(x_{t+1}) = \mathbb{E}_{a_{t+1} \sim \pi} \big[\Bar{Q}(x_{t+1}, a_{t+1}) - \alpha  \log \pi(a_{t+1}|x_{t+1}) \big],
\end{equation*}
where $\bar{Q}$ is the target soft Q-function parameterized by a weight vector obtained from an exponentially moving average of the Q-function weights to stabilize training. The  policy improvement step then attempts to project a parametric policy $\pi(a_t|x_t)$  by minimizing KL divergence between the  policy and a Boltzmann distribution induced by the Q-function, producing the following objective,
\begin{equation*}
    J(\pi)= \mathbb{E}_{x_t \sim \mathcal{D}} \bigg[ \mathbb{E}_{a_t \sim \pi} [\alpha \log (\pi(a_t | x_t)) - Q(x_t, a_t)] \bigg].
\end{equation*}

We provide the hyperparameters used for the RL experiments in \cref{table:rl_hyper_params}.

\begin{table}[hb!]
\centering
\begin{tabular}{|l|c|}
\hline
Parameter name        & Value \\
\hline
Replay buffer capacity & $1000000$ \\
Batch size & $1024$ \\
Discount $\gamma$ & $0.99$ \\
Optimizer & Adam \\
Critic learning rate & $10^{-5}$ \\
Critic target update frequency & $2$ \\
Critic Q-function soft-update rate $\tau_{\textrm{Q}}$ & 0.005 \\
Critic encoder soft-update rate $\tau_{\textrm{enc}}$ & 0.005 \\
Actor learning rate & $10^{-5}$ \\
Actor update frequency & $2$ \\
Actor log stddev bounds & $[-5, 2]$ \\
Encoder learning rate & $10^{-5}$ \\
Decoder learning rate & $10^{-5}$ \\
Decoder weight  decay & $10^{-7}$  \\
L1 regularization weight & $10^{-5}$ \\
Temperature learning rate & $10^{-4}$ \\
Temperature Adam's $\beta_1$ & $0.9$ \\
Init temperature & $0.1$ \\
\hline
\end{tabular}\\
\caption{\label{table:rl_hyper_params} A complete overview of used hyper parameters.}
\end{table}

\section{Additional Results}
\subsection{Reinforcement Learning}
We find that even without noise on the ground truth states, with only two environments, baseline SAC fails (\cref{fig:cartpole_swingup_rl_noise0}).

\begin{figure}[h]
    \centering
    \includegraphics[width=0.3\textwidth]{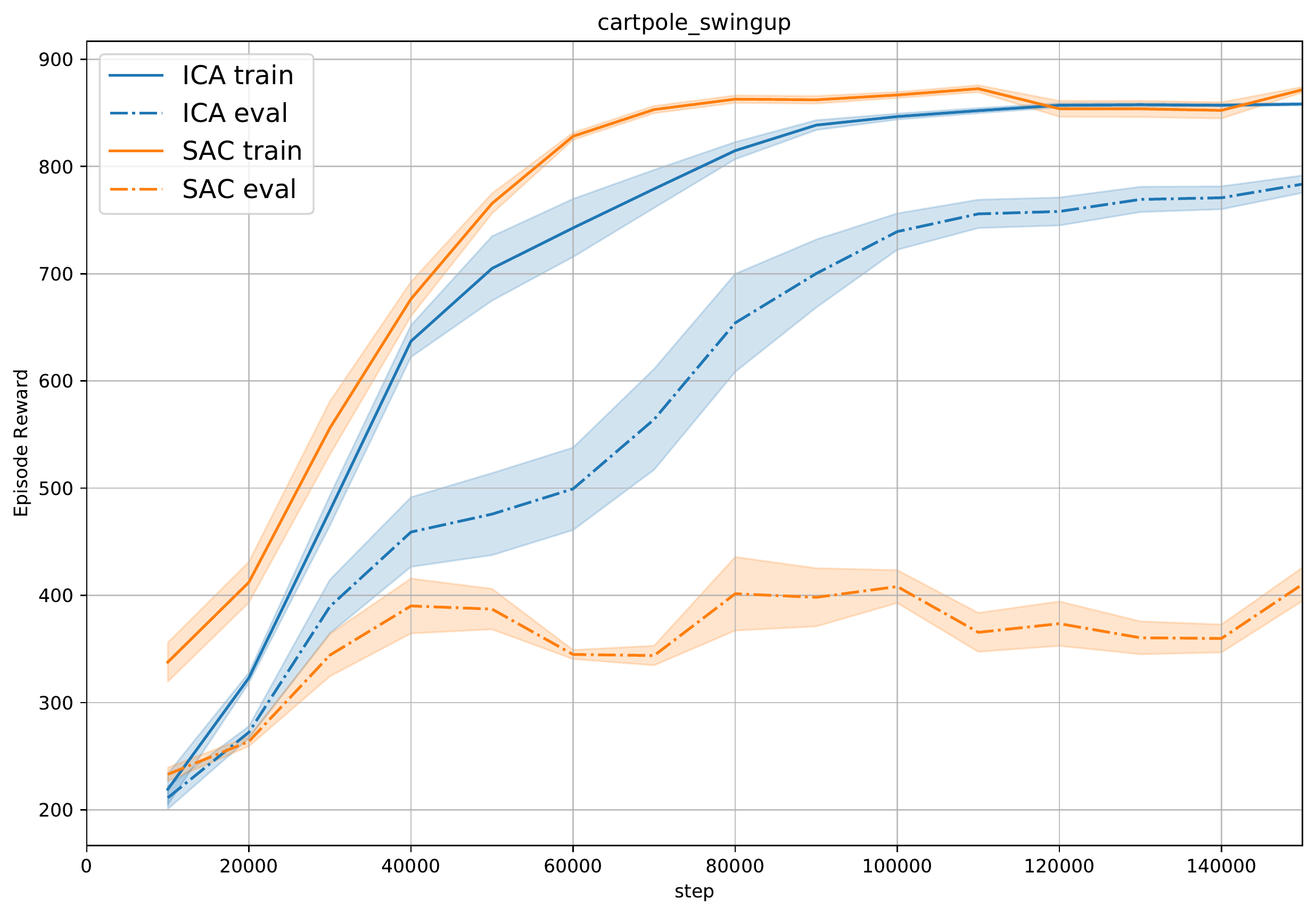}
    \caption{Generalization gap in SAC performance with 2 training environments on Cartpole Swingup from DMC. Evaluated with 10 seeds, standard error shaded.}
    \label{fig:cartpole_swingup_rl_noise0}
\end{figure}
\end{document}